\documentclass[a4paper, 12pt, oneside]{amsart}
\usepackage{graphicx}
\usepackage{covington}
\usepackage{eulervm}
\usepackage{enumerate}
\usepackage{centernot}
\usepackage{amsfonts}
\usepackage{amsmath}
\usepackage{amssymb}
\usepackage{amsthm}
\usepackage{bussproofs}
\usepackage{tikz-cd}
\usepackage{hyperref}
\hypersetup{
  colorlinks = true,
  citecolor = {purple!70!black},
  linkcolor = {green!70!black},
  urlcolor = {blue!80!black}
}
\usepackage{import}
\usepackage{thmtools}
\usepackage{thm-restate}
\declaretheoremstyle[
headformat=\NAME\,\NUMBER\NOTE,
postheadspace=.5em,
spaceabove=6pt,
headfont=\bfseries\small,
notefont=\normalfont\small\mdseries, notebraces={(}{)},
bodyfont=\normalfont\itshape
]{plainswap}
\declaretheorem[style=plainswap, name=Theorem, sharenumber=subsection]{theorem}

\declaretheorem[style=plainswap, numberlike=theorem, name=Corollary]{corollary}

\declaretheoremstyle[
headformat=\NAME\,\NUMBER\NOTE,
postheadspace=.5em,
spaceabove=6pt,
headfont=\bfseries\small,
notefont=\normalfont\mdseries, notebraces={(}{)},
bodyfont=\normalfont
]{definitionswap}
\theoremstyle{definition}
\declaretheorem[style=definitionswap, numberlike=theorem, name=Definition]{definition}
\declaretheorem[style=definitionswap, numberlike=theorem, name=Example]{eg}
\newtheorem*{remark}{Remark}
\newcommand{\Z}{\mathbb{Z}}

\renewcommand{\P}{\mathcal{P}}
\begin{document}
%%%%%MAIN DETAILS ABOUT THE AUTHOR
\title{$\pi$-augmented pregroups and applications to linguistics}
\author{Valentin Boboc}
\address{School of Mathematics, The University of Manchester, Alan Turing Building, Oxford Road, Manchester M13 9PL, United Kingdom}
\email{valentinboboc@icloud.com}

%%%%%ABSTRACT
\begin{abstract}
We enrich pregroups with a mapping which allows us to locally apply precyclic permutations to designated substrings. We prove a normalisation theorem for such algebraic structures and briefly formalise some known applications of pregroups to the analysis of clitic pronouns in certain natural languages. 
\end{abstract}
\maketitle
%%%%%TABLE OF CONTENTS (uncomment the two lines below)
%\setcounter{tocdepth}{1}
%\tableofcontents
%%%%%%MAIN CONTENTS OF THE PAPER
\section{Introduction} The purpose of language is to convey meaning. This article focuses on the compositional aspect of language, namely the understanding that the meaning of a longer phrase is determined by its structure together with the meanings of its constituent parts. 

A good mathematical model for the compositional aspect of language should be capable to evaluate whether a sentence is well-formed (grammatically acceptable) by calculating the overall grammatical type of a phrase by ``composing'' the types of the sentence's individual parts. 

In \cite{lambek1997type}\cite{lambek2008wordtosent} Lambek introduced the concept of a pregroup to model the syntax of sentence structure in natural languages. Pregroups have been applied to the analysis of multiple natural languages, from English to French, German, Farsi, Latin, and Japanese (see e.g. \cite{lambek2003german}\cite{bargelli2001french}\cite{casadio2005latin}\cite{sadrzadeh2007persian}\cite{cardinal2002algebraic}). Most work involving applications to linguistics employs free pregroups, which can also be interpreted as compact closed monoidal categories \cite{coecke2010mathematical}\cite{kelly1980coherence}. 

We briefly recall the algebraic machinery. 
\begin{definition}
A \emph{pregroup} is a tuple $(P, \cdot, 1, -^\ell, -^r, \leq)$ where $(P, \cdot, 1, \leq)$ is a partially ordered monoid and the unary operations $-^\ell, -^r$ (the left and the right adjoint) satisfy for all $x \in P$ the following relations:
\begin{align*}
    &\text{(contractions)}\quad \quad \quad   x \cdot x^r \leq 1 \quad \quad \quad \quad x^\ell \cdot x \leq 1\\
    &\text{(expansions)}\quad \quad \quad \ \ 1 \leq x^r \cdot x \quad \quad \quad \quad 1 \leq x \cdot x^\ell.
\end{align*}
\end{definition}

We define $p^{(n)}$ inductively as $p^{(0)} = p$, $(p^{(n)})^r = p^{(n+1)}$, and $(p^{(n)})^\ell = p^{(n-1)}$. An ordinary element $p$ of a pregroup $P$ is called an \emph{atom} (or \emph{type}), an element of the form $p^{(n)}$ is called a \emph{term}, and a concatenation of terms $X$ is called a \emph{string} (or \emph{word}).

It is immediate to check that the following relations hold in every pregroup:
\begin{align*}
    &1^\ell = 1 = 1^r
    &&(x^\ell)^r = x = (x^r)^\ell\\
    &(xy)^\ell = y^\ell x^\ell \text{ and } (xy)^r = y^rx^r
    &&\text{if } x\leq y \text{ then } y^\ell \leq x^\ell \text{ and } y^r \leq x^r. 
\end{align*}

\begin{eg}
The simplest example of a pregroup is that of an ordered group $G$. In this instance, for every element $x \in G$, we have $x^\ell = x^r = x^{-1}$. The left and right adjoint operations collapse into the group inverse and so a lot of the structure we are interested in is absent in this case.
\end{eg}

\begin{eg}
In \cite{lambek1997type} we have the following example of a pregroup which is not an ordered group: take $P$ to be the set of all unbounded, monotone functions from the integers into itself. Let $[x]$ denote the greatest integer $m$ such that $m \leq x$. For $f(x) = 2x$ we define the left adjoint to be $f^\ell(x) = [(x+1)/2]$ and the right adjoint to be $f^r(x) = [x/2]$. Thus we see that $f^\ell \neq f^r$.
\end{eg}

For the purpose of studying the grammar of natural languages, we are mostly interested in free pregroups, i.e. ordered monoids with left and right adjoints freely generated by a finite set of atoms or grammatical types. 

We give a toy example in the English language. 

\begin{eg}\label{example E}
Let $E$ be the pregroup freely generated by the set $\{n, s\}$, where we think of $n$ as the atom that denotes nouns and $s$ the atom that denotes a sentence in the English language. Let's analyse the sentence ``Cats eat mice.'' Since ``cats'' and ``mice'' are both nouns, they are assigned the type $n$. The natural word order in English is SVO (subject -- verb -- object) and thus in order for the verb ``eat'' to generate a well-formed sentence it requires a subject to its left and an object to its right. In our model, we assign the type $n^r s n^\ell$. The overall sentence then has the type $n (n^r s n^\ell) n$ and we can use the axioms of pregroups to make the following derivation: \[n (n^r s n^\ell) n = (n n^r) s (n^\ell n) \leq s (n^\ell n) \leq s.\]

Here, we used the associativity of the monoid $E$ together with two contractions to determine that the sentence ``Cats eat mice'' reduces to the sentence type $s$ and it is thus well-formed. 
\end{eg}

\begin{eg}
    Using the pregroup $E$ from the previous example, we model the noun phrase ``black cats.'' We take into account the fact that ``cats'' is already a noun and should be assigned the type $n$. The whole phrase ``black cats'' is a noun phrase and should ideally reduce to the noun type. Thus the adjective ``black'' should naturally be assigned the type $nn^\ell$. This is indeed a reasonable rule for all adjectives and adjective phrases in the English language. The reduction of ``black cats'' is then straightfoward: \[(nn^\ell)n  = n(n^\ell n) \leq n.\] 

\end{eg}
These types of calculations arise naturally from the setup and they tie in with the word problem for pregroups. Given a pregroup $P$, we can formulate the word problem as follows: \[\text{for given strings }X,Y \in P\text{ verify whether }X \leq Y.\] 

For example, solving this word problem $X \leq s$ is equivalent to verifying whether a given sentence or phrase is grammatically acceptable in our model of natural language grammar. The following result is crucial in understanding why it is possible for us to check if a string reduces to the sentence type.

\begin{theorem}[The Switching Lemma, \cite{lambek1997type}]
Given two strings $X, Y$ in a free pregroup $P$ and $X \leq Y$, there exists a string $U$ such that $X \leq U$ can be derived without using expansions, and $U \leq Y$ can be derived without using contractions.
\end{theorem}

In other words, upon deriving $X \to Y$ in a freely generated pregroup we may assume without loss of generality that this can be derived in a sequence of steps where no term introduced by an expansion is later on cancelled by a contraction. Thus if we set $X=p$ where $p$ is an atom, then $p \to Y$ can be deduced by means of expansions only. On the other hand, if we set $Y = q$ for some atom $q$ then we can deduce $X \to q$ by means of contractions only. 

The second case is the interesting case for grammarians because the grammatical acceptability $X \leq s$ of a sentence represented by some word $X$ can be verified by contractions only. Another consequence of the Switching Lemma is that the word problem for free pregroups can be solved in polynomial time.

Several ways of decorating pregroups with additional structures have been considered. The reader is encouraged to consult the linguistics focused discussions in \cite{kislak2012extended}\cite{lambek2007should}. A prominent decoration is Fadda's $\beta$-pregroups (or pregroups with modalities) \cite{fadda2002towards}\cite{kislak2007logic}, where pregroups are decorated with an additional operator $\beta$ which locally suppresses the associativity of the underlying monoidal operation. 

An excellent example that demonstrates the usefulness of $\beta$-pregroups is the solution to the following ambiguous statement. 

\begin{eg}
The phrase ``old teachers and students'' can be interpreted in two distinct ways. We may assume that the adjective ``old'' distributes to the coordinated nouns and thus both ``teachers'' and ``students'' are understood to be ``old.'' Alternatively, we may assume that the adjective ``old'' only modifies the proximate noun ``teachers.'' In this situation, we first obtain the noun phrase ``old teachers'' and only afterwards coordinate that with the unmodified noun ``students.'' Using the pregroup grammar $E$ from Example \ref{example E}, the parsing ``old (teachers and students)'' is obtained by the following reduction:
\[(nn^\ell) (\underline{n n^r}n\underline{n^\ell n}) \leq (n\underline{n^\ell) n} \leq n.\]

The underlined letters indicate the places where contractions were applied. The alternative interpretation of ``(old teachers) and students'' is obtained in the following way:
\[((n\underline{n^\ell)n})(n^rnn^\ell)n\leq \underline{n(n^r} n \underline{n^\ell)n} \leq n.\]

Both of these methods of parsing yield grammatically acceptable noun phrases. The only difference consists of the way in which we choose to use the associativity of the pregroup. Observing the fact that pregroup grammars can naturally incorporate parsing ambiguity is encouraging as all natural languages have in-built ambiguity and redundancy (see e.g. \cite{piantadosi2012communicative}\cite{tal2022redundancy}). $\beta$-pregroups give us a method to exclude either of the two interpretations and choose to fix a default interpretation. For instance, if we want to be clear about the fact that the adjective should always distribute as a modifier to all coordinated noun phrases, we can write the sentence in the corresponding $\beta$-pregroup as:

\[nn^\ell \beta(n) \beta(n)^r n \beta(n)^\ell \beta(n).\]

In the $\beta$-pregroup $E$, some atoms can be decorated with the symbol $\beta$. More formally, $\beta$ is a monotone morphism of monoids which locally suppresses associativity. Note that we don't need to employ paranthesis any longer because the operator $\beta$ forces all contractions between decorated atoms to occur independently of contractions between undecorated atoms. There is exactly one way to reduce this string to the noun type:
\begin{align*}
    nn^\ell \underline{\beta(n) \beta(n)^r} n \beta(n)^\ell \beta(n) &\leq nn^\ell n \underline{\beta(n)^\ell \beta(n)}\\
    &\leq n\underline{n^\ell n}\\
    &\leq n.
\end{align*}

Note that the order of some of the contractions is not relevant, but we no longer have a way to parse this string without first coordinating the nouns and then applying the adjective to the whole noun phrase. Without going through all the details, making the adjective modify only the proximate noun can be accomplished with the following string in the $\beta$-pregroup $E$:

\[n\beta(n)^\ell \beta(n) n^r n n^\ell n.\]
\end{eg}

A grammatical aspect that seems complicated to model using pregroups is \emph{clitic movement}. In brief, a clitic is a morpheme which has the syntactic properties of a word but is phonologically dependent on another word in the sense that it is affixed to a host in order to play a syntactic role. In English, the contracted forms of auxilliary verbs such as ``I\textbf{'m},'' ``he\textbf{'s},'', and ``we\textbf{'ve}'' are examples of clitics. 

Any part of speech can be cliticised. In English, auxilliary verbs are common clitics. In Latin, the conjunction ``que'' (\textit{and}) is cliticised in the phrase ``Senatus Populus\textbf{que} Romanus'' (\textit{The Senate and People of Rome}). In Romanian, the definite articles ``i'' and ``le'' are always enclitic as in ``copii\textbf{i}'' (\textbf{the} \textit{children}) or ``fete\textbf{le}'' (\textbf{the} \textit{girls}). 

The most common part of speech to feature as clitics is the personal pronoun. In many languages the cliticisation of the pronoun is accompanied by a phenomenon called \emph{clitic movement}. When a part of speech is replaced by a clitic pronoun, this often changes the word order. 

\begin{eg}
In French, the usual word order is SVO (subject-verb-object) as in the sample sentence: \begin{align*}
    \text{Emmanuel a lu } \textbf{les lettres.  }\\
    \textit{Emmanuel has read} \textbf{ the letters.} 
\end{align*}

If we replace the object noun ``lettres'' with a clitic pronoun then the worder becomes SOV:

\begin{align*}
    \text{Emmanuel} \textbf{ les } \text{a lues.  }\\
    \textit{Emmanuel} \textbf{ them } \text{has read.}
\end{align*}
\end{eg}

To account for clitic movement, some authors have proposed introducing into the pregroup new atoms  for verbal inflectors and infinitives \cite{bargelli2001french}. For instance, verbs which take direct and indirect objects are assigned simultaneously two different types: an extended infinitive of type $j$ and a short infinitive of type $i$. Clitic pronouns are then assigned rather lengthy strings where the authors creatively use the fact that pregroups allow for iterated adjoints. 

A more simple approach is given in \cite{casadio2009clitic}, where the author introduces metarules (informal relations in the pregroup) which account for movements within the sentence. One such rule reads that if a verb has type $qp^\ell$ then it also has type $\overline{p}^r q$, where the bar $\overline{p}$ indicates a cliticised pronoun. 

Continuing our example in French, we can consider a pregroup $F$ freely generated by $\{n,s,p,\overline{p}\}$ with $n \leq p$. The sentence ``Emmanuel a lu les letres'' has type $n (n^r s n^\ell) n$. The second sentence where we cliticise the object ``Emmanuel les a lues'' has type $n \overline{p} (\overline{p}^r n^r s)$. One can easily verify that they both reduce to the sentence type $s$ and are thus grammatically acceptable. 

This metarule is a consequence of applying what is called a \emph{precylic permutation}, which reminds of Yetter's cyclic linear logic \cite{yetter1990quantales}. Left and right precyclic permutations are given, respectively, by the following axioms:
\[qp \leq pq^{\ell\ell} \quad \quad \text{and} \quad \quad qp \leq p^{rr}q.\]

The clitic metarules can be derived using these axioms. Unfortunately, incorporating either of these axioms into the axioms of a pregroup will cause the pregroup to collapse into an ordered group, thus losing vital structure for our grammatical model \cite[Section 6]{casadio2009clitic}. Some details of such precyclic rules have already been explored in the literature \cite{kislak2012extended}\cite{debnath2019pregroup}\cite{casadio9tupled}\cite{stabler2008tupled} as well as different pregroup approaches to clitic movement \cite{casadio2007applying}\cite{casadio2008recent}\cite{lambek2010exploring}\cite{casadio2010agreement}. 

The purpose of this paper is to introduce a new decoration which allows precyclic permutations in a local way. In other words, we want to formalise the clitic metarules as an algebraic structure in a way similar to $\beta$-pregroups where we can locally allow precyclicity (as opposed to locally forbidding associativity). This approach would create a pathway for handling such movements without metarules or relations but rather by means of an internal feature of decorated pregroups. 

The article is organised as follows. In Section \ref{Section pregroups}, we define $\pi$-augmented pregroups as a free algebraic structure with certain rewriting rules. Section \ref{section 3} deals with the statement and proof of a normalisation theorem for $\pi$-augmented pregroups. This is the equivalent of the traditional Switching Lemma and it tells us that the word problem for $\pi$-augmented pregroups can be solved in the same fashion. So $\pi$-augmented pregroups are a reasonable extension of ordinary pregroups. In Section \ref{section 4}, we apply the new mathematical machinery to analyse some examples of clitic movement. We explore both old examples from the literature to showcase the power of $\pi$-augmented pregroups and some new examples together with an array of questions and phenomena left to be explored in the future.

\section{\texorpdfstring{$\pi$}{TEXT}-augmented pregroups}\label{Section pregroups}
We extend the definition of pregroup to capture the property of what we call \emph{local precyclicity}. For the purpose of this article, we only focus on left precyclicity. 
\begin{definition}
A \emph{$\pi$-augmented pregroup} is a pregroup $\mathcal{P}$ together with a monotone mapping $\pi : \mathcal{P} \to \mathcal{P}$ such that atoms acted on by $\pi$ obey the left precyclic rule, i.e. for any $ 1\neq p,q \in \mathcal{P}$ we have $[\pi(p)][\pi(q)]^\ell \to [\pi(q)]^r[\pi(p)]$.  
\end{definition}

We next show how to construct a free $\pi$-augmented pregroup $\mathcal{P}$. Start with a non-empty partially ordered set $(P, \leq)$. The elements of $\mathcal{P}$ are defined by induction:

\begin{enumerate}[leftmargin=1em]
    \item the empty word $1$ is in  $\P$,
    \item $p^{(n)}$ is in $\P$ for all $p \in P$ and all $n \in \Z$,
    \item if $w_1, w_2$ are in $\P$ then the concatenation $w_1w_2$ is in $\P$,
    \item if a word $w$ is in $\P$ then $[\pi(w)]^{(n)}$ is in $\P$ for all $n \in \Z$. 
\end{enumerate}

The usual conventions regarding left and right adjoints apply. Namely, $p^{(0)} = p$, $[p^{(n)}]^r = p^{(n+1)}$, $[p^{(n)}]^\ell = p^{(n-1)}$, $[\pi(p)]^{(n)r} = [\pi(p)]^{(n+1)}$, and $[\pi(p)]^{(n)\ell} = [\pi(p)]^{(n-1)}$. Additionally, the left and right adjoints act on concatenations by reversing the order, i.e. \[(p_1 p_2\ldots p_k)^\ell = (p_k)^\ell \ldots (p_2)^\ell (p_1)^\ell\] and similarly for the right adjoint. 

In the sequel, $p, q \in \mathcal{P}$ are atoms and $X, Y, Y', Z \in \P$ are non-empty words. The rewriting rules are given in the following list.

\begin{enumerate}
    \item \textbf{Contraction rules}
        \begin{enumerate}[leftmargin=1em]
            \item[(CON)] - simple contraction:\\
            $X, p^{(n)}, p^{(n+1)}, Y \to X, Y$;
            \item[($\pi$-CON)] - augmented contraction:\\
            $X, [\pi(p)]^{(n)}, [\pi(p)]^{(n+1)}, Y \to X, Y$;
            \item[(IND-C)] - induced contraction:\\
            $X, [\pi(Y)]^{(2n)} ,Z \to X, [\pi(Y')]^{(2n)}, Z$\\
            where $Y \to Y'$ is a contraction rule;\\
            $X, [\pi(Y')]^{(2n+1)} ,Z \to X, [\pi(Y)]^{(2n+1)}, Z$\\
            where $Y\to Y'$ is an expansion rule.
        \end{enumerate}
    \item \textbf{Expansion rules}
        \begin{enumerate}[leftmargin=1em]
            \item[(EXP)] - simple expansion:\\
            $X, Y \to X, p^{(n+1)}, p^{(n)}, Y$;
            \item[($\pi$-EXP)]- augmented expansion:\\
            $X, Y \to X, [\pi(p)]^{(n+1)}, [\pi(p)]^{(n)}, Y$;
            \item[(IND-E)] - induced expansion:\\
            $ X, [\pi(Y)]^{(2n)}, Z \to X, [\pi(Y')]^{(2n)} ,Z$\\
            where $Y \to Y'$ is an expansion rule;\\
            $ X, [\pi(Y')]^{(2n+1)}, Z \to X, [\pi(Y)]^{(2n+1)} ,Z$\\
            where $Y \to Y'$ is a contraction rule.
        \end{enumerate}
    \item \textbf{Special rules} (neither contraction, nor expansion)
        \begin{enumerate}[leftmargin=1em]
            \item[(IND)] - simple induced step:\\
            $X, p^{(2n)}, Y \to X, q^{(2n)}, Y$;\\
            $X, q^{(2n+1)}, Y \to X, p^{(2n+1)}, Y$\\ where $p\leq q$ in $P$.
            \item[($\pi$-IND)] - augmented induced step:\\
            $X, [\pi(Y)]^{(2n)}, Z \to X, [\pi(Y')]^{(2n)}, Z$;\\
            $X, [\pi(Y')]^{(2n+1)}, Z \to X, [\pi(Y)]^{(2n+1)}, Z$\\
            where $Y\to Y'$ is a special rule.
            \item[(m-IND)] - mixed induced step:\\
            $X, Y^{(2n)}, Z \to X, [\pi(Y')]^{(2n)}, Z$;\\
            $X, [\pi(Y')]^{(2n+1)}, Z \to X, Y^{(2n+1)}, Z$\\
            where $Y \to Y'$ is a special rule.
            \item[(PRE)] - precyclic rule:\\
            $X, [\pi(Y)], [\pi(Y')]^\ell, Z \to X, [\pi(Y')]^r, [\pi(Y)], Z$.
        \end{enumerate}
\end{enumerate}

Next define $\Rightarrow$ to be the reflexive and transitive closure of $\to$. On $\P$ there is an equivalence relation on strings: $X \sim Y$ if and only if $X \Rightarrow Y$ and $Y \Rightarrow X$. Equivalence classes of strings $[X]$ can be endowed with a monoid operation: $[X] \cdot [Y] = [X,Y]$, a left adjoint unary operation: $[X]^\ell = [X^\ell]$, a right adjoint unary operation: $[X]^r = [X^r]$, an identity string $1$ equal to the equivalence class of the empty word, and a preorder: $[X] \leq [Y]$ if and only if $X \Rightarrow Y$. 

We have thus endowed $\P$ with a pregroup structure. $\P$ is then called a \emph{free $\pi$-augmented pregroup}. 

\begin{remark}
    If we were to define ordinary pregroups in this way, then following the same procedure, the rewriting rules would be given by a much shorter list:
    \begin{enumerate}[leftmargin=1em]
        \item[(CON)] $X, p^{(n)}, p^{(n+1)}, Y \to X, Y$;
        \item[(EXP)] $X, Y \to X, p^{(n+1)}, p^{(n)}, Y$;
        \item[(IND)] $X, p^{(n)}, Y \to X, q^{(n)}, Y$ if $p \leq q$ ($q \leq p$) and $n$ is even (odd).
    \end{enumerate}

This rewriting system was Lambek's original formulation of the logic of pregroups, otherwise known as Compact Bilinear Logic (see e.g. \cite{buszkowski2003sequent}\cite{casadio2002tale}\cite{lambek2012logic}). 
\end{remark}
\begin{remark}
We make a quick observation about the mixed induced step (m-IND). Since the relation on $\mathcal{P}$ is reflexive, that means we have $p \to p$ for every atom $p \in \mathcal{P}$. Applying (m-IND), we deduce that $p^{(2n)} \to [\pi(p)]^{(2n)}$ and $[\pi(p)]^{(2n+1)}\to p^{(2n+1)}$ for all $n \in \mathbb{Z}$. The given rewriting rules let us conclude that it is ocasionally possible to contract mixed terms in the following way:
\[p^{(2n)} [\pi(p)]^{(2n+1)} \to 1 \quad \quad \text{and}\quad \quad [\pi(p)]^{(2n-1)}p^{(2n)} \to 1.\]

Both of these ``mixed contractions'' can be derived by applying (m-IND) followed by (CON). Note however that $[\pi(p)]\ p^r \centernot\to 1$ as (m-IND) cannot be applied to either $\pi(p)$ or $p^r$. Henceforth, whenever we apply (m-IND) followed by (CON) to contract an eligible mixed term, we will denote this combined step as (m-CON). 
\end{remark}
\section{Normalisation theorem}\label{section 3}
In this section we prove a normalisation theorem for free $\pi$-augmented pregroups akin to the Switching Lemma for free pregroups. In other words we are interested in showing that in any derivation of $X \Rightarrow Y$, all contraction rules may occur before all expansion rules. We proceed to formalise and prove this statement.
\begin{definition}
A derivation $X \Rightarrow Z$ is called normal if there exists some string $Y$ such that $X \Rightarrow Z$ can be expressed as the composition of derivations $X \Rightarrow Y$ and $Y \Rightarrow Z$, where $X \Rightarrow Y$ is derived without expansion rules and $Y \Rightarrow Z$ is derived without contraction rules.
\end{definition}
\begin{theorem}
In a $\pi$-augmented pregroup, every derivation of $X \Rightarrow Y$ of minimal length is a normal derivation.
\end{theorem}
\begin{proof}
Let $X_0 \to X_1 \to \ldots \to X_n$ be a derivation of $X \Rightarrow Y$ of length $N$ such that $X = X_0$, $Y = X_N$, and every step $X_{i-1} \to X_i$ is one of the rewriting rules of a $\pi$-augmented pregroup. We show that every such derivation can be written as a normal derivation of length at most $N$. 

We proceed by induction on $N$. For $N=0$ and $N=1$, the statement is clearly true. Now for $N\geq 2$, the derivation $X_1 \to X_2 \to \ldots \to X_N$ is shorter than the initial one and it can thus be transformed into a normal derivation $Y_1 \to Y_2 \to \ldots \to Y_m$ where $Y_1 = X_1$, $Y_m = X_N$, and $m \leq N$. Suppose that $m < N$. Then $X_0 \to Y_1 \to Y_2 \to \ldots \to Y_m$ is a derivation of $X \Rightarrow Y$ of length less than $N$, so by the induction hypothesis this can be transformed into a normal derivation. 

It remains to deal with the situation when $N = m$. We are interested in the nature of the first step $X_0 \to X_1$. Problems may only arise if this step is an expansion rule. We focus on the portion of the derivation $X_0 \to X_1 \ldots \to X_\ell$ where $X_0 \to X_1$ is an expansion rule and $X_{\ell-1} \to X_\ell$ is the first contraction rule. If the contraction rule occurs immediately after the expansion rule and they are independent, then we may switch the order in which they occur in the derivation and then we are done. If any of the special rules that occur between $X_0 \to X_1$ and $X_{\ell-1} \to X_\ell$ are independent of these two steps then we may move them either before $X_0 \to X_1$ or after $X_{\ell-1} \to X_\ell$, which renders a contradiction to minimality. So we assume that all steps are dependent in the portion of the derivation that we analyse. 

The rule applied to the first step $X_0 \to X_1$ is denoted $A_0$ and the rule applied to step $X_{\ell-1} \to X_\ell$ is denoted $A_1$. The following cases need to be analysed:

\begin{enumerate}[leftmargin=1em]
    \item[II.1.] $A_0=$(EXP), $A_1=$(CON)
    \item[II.2.] $A_0=$(EXP), $A_1=$($\pi$-CON)
    \item[II.3.] $A_0=$(EXP), $A_1=$(IND-C)
    \item[III.1.] $A_0=$($\pi$-EXP), $A_1=$(CON)
    \item[III.2.] $A_0=$($\pi$-EXP), $A_1=$($\pi$-CON)
    \item[III.3.] $A_0=$($\pi$-EXP), $A_1=$(IND-C)
    \item[IV.1.] $A_0=$(IND-E), $A_1=$(CON)
    \item[IV.2.] $A_0=$(IND-E), $A_1=$($\pi$-CON)
    \item[IV.3.] $A_0=$(IND-E), $A_1=$(IND-C)
\end{enumerate}

We present some sample analyses.
\begin{enumerate}
    \item[Case I.] $X_0 \to X_1$ is a contraction rule. Then $X_0 \to Y_1 \to \ldots \to Y_m$ is a normal derivation of $X \Rightarrow Y$ and we are done.
    \item[Case II.] $X_0 \to X_1$ is an (EXP) rule, say, $U, V \to U, p^{(n+1)}, p^{(n)}, V$. 
    \item[Case II.1.] Suppose there are no contraction rules being applied to terms of the form $p^{(n+1)}, p^{(n)}$ in $Y_1 \to Y_2 \to \ldots \to Y_m$. Then we skip the first step and drop all the terms $p^{(n+1)}, p^{(n)}$ from the types appearing in the contraction rules of $Y_1 \to Y_2 \to \ldots \to Y_m$. Then add the terms $p^{(n+1)}, p^{(n)}$ back with a single instance of (EXP), then continue on with the expansion rules in the sequence. This yields a normal derivation of length $N$.
    \item[Case II.2.] Suppose there is a (CON) step applied to the term $p^{(n)}$ later on in the derivation. Such a derivation can have the form
    \begin{align*}
        U, q_0^{(2n)}, V &\to U, q_0^{(2n)}, q_k^{(2n+1)}, q_k^{(2n)}, V &&(EXP)\\
        &\to U, q_0^{(2n)}, q_{k-1}^{(2n+1)}, q_k^{(2n)}, V &&(IND)\\
        &\to U, q_0^{(2n)}, q_{k-2}^{(2n+1)}, q_k^{(2n)}, V &&(IND)\\
        &\to \cdots\\
        &\to U, q_0^{(2n)}, q_0^{(2n+1)}, q_k^{(2n)}, V &&(IND)\\
        &\to U, q_k^{(2n)}, V &&(CON)
    \end{align*}
    where $q_0 \leq q_1 \leq \ldots \leq q_k$ holds in $P$. This is a derivation of length $k+2$ consisting of one (EXP) step, one (CON) step and $k$ (IND) steps. 
    
    We can transform this derivation in the following manner.
    \begin{align*}
        U, q_0^{(2n)}, V &\to U, q_1^{(2n)}, V &&(IND)\\
        &\to U, q_2^{(2n)}, V &&(IND)\\
        &\to \cdots\\
        &\to U, q_k^{(2n)}, V &&(IND)
    \end{align*}
    This derivation has no expansions occurring before contractions and has length $k$, which makes it shorter than the original derivation. Combining this observation with the induction hypothesis, we obtain a normal derivation. 
    \item[Case III.] $X_0 \to X_1$ is a ($\pi$-EXP) rule, say $U, V \to U, [\pi(p)]^{(n+1)}, [\pi(p)]^{(n)}, V$. 
    \item[Case III.1.] Suppose there is a (CON) step later on in the sequence. The relevant part of the derivation can take the form:
        \begin{align*}
            U, q_0^{(2n)}, V &\to U, q_0^{(2n)}, [\pi(q_k)]^{(2n+1)}, [\pi(q_k)]^{(2n)}, V &&(\pi-EXP)\\
            &\to U, q_0^{(2n)}, [\pi(q_{k-1})]^{(2n+1)}, [\pi(q_k)]^{(2n)}, V &&(\pi-IND)\\
            &\to U, q_0^{(2n)}, [\pi(q_{k-2})]^{(2n+1)}, [\pi(q_k)]^{(2n)}, V &&(\pi-IND)\\
            &\to \cdots\\
            &\to U, q_0^{(2n)}, q_0^{(2n+1)}, [\pi(q_k)]^{(2n)}, V &&(m-IND)\\
            &\to U, [\pi(q_k)]^{(2n)}, V &&(CON)
        \end{align*}
    where $q_0\leq q_1 \leq \ldots \leq q_k$ in $P$. This derivation has length $k+2$ and can be transformed to:
        \begin{align*}
            U, q_0^{(2n)}, V &\to U, q_1^{(2n)}, V &&(IND)\\
            &\to U, q_2^{(2n)}, V &&(IND)\\
            &\to \cdots\\
            &\to U, q_{k-1}^{(2n)}, V &&(IND)\\
            &\to U, \pi(q_k)^{(2n)}, V &&(m-IND)
        \end{align*}
    The new derivation has length $k$ and only uses induced steps. So we have a normal derivation.
    \item[Case III.2.] Suppose there is a ($\pi$-CON) step which cancels out one of the terms produced in the ($\pi$-EXP) step. The relevant portion of the derivation can take the shape:
        \begin{align*}
        &U, [\pi(q_0)]^{(2n)}, V\\
        &\to U, [\pi(q_0)]^{(2n)}, [\pi(q_k)]^{(2n+1)}, [\pi(q_k)]^{(2n)}, V &&(\pi-EXP)\\
        &\to U, [\pi(q_0)]^{(2n)}, [\pi(q_{k-1})]^{(2n+1)}, [\pi(q_k)]^{(2n)}, V &&(\pi-IND)\\
        &\to U, [\pi(q_0)]^{(2n)}, [\pi(q_{k-2})]^{(2n+1)}, [\pi(q_k)]^{(2n)}, V &&(\pi-IND)\\
        &\to \cdots\\
        &\to U, [\pi(q_0)]^{(2n)}, [\pi(q_0)]^{(2n+1)}, [\pi(q_k)]^{(2n)}, V &&(\pi-IND)\\
        &\to U, \pi(q_k)^{(2n)}, V &&(\pi-CON)
        \end{align*}
        where once again $q_0 \leq q_1 \leq \ldots \leq q_k$ holds in $P$. This is a derivation of length $k+2$. Any (PRE) steps that are necessary are performed after ($\pi$-EXP) and before ($\pi$-CON) by assumption. This can again be shortened to:
        \begin{align*}
            U, [\pi(q_0)]^{(2n)}, V &\to U, [\pi(q_1)]^{(2n)}, V &&(\pi-IND)\\
            &\to U, [\pi(q_2)]^{(2n)}, V &&(\pi-IND)\\
            &\to \cdots\\
            &\to U, [\pi(q_k)]^{(2n)}, V &&(\pi-IND)
        \end{align*}
        This derivation consists of $k$ ($\pi$-IND) steps. If any (PRE) rules are used these may be assumed to occur at the very beginning without affecting the final result. Using the induction hypothesis again, we obtain that such a minimal derivation of $X \Rightarrow Y$ must be normal.
    \item[Case III.3.] Suppose there is a (IND-C) step at the end of the sequence $X_0\to X_1$. This means that we start with a derivation of the form \[U, V \to U, [\pi(p)]^{(2n+1)}, [\pi(p)]^{(2n)}, V,\] which is the ($\pi$-EXP) step. This then concludes in $X_1$ with a derivation of the form \[U', [\pi(p)]^{(2n+1)}, [\pi(p)]^{(2n)}, V' \to U', [\pi(p)]^{(2n+1)}, [\pi(q)]^{(2n)}, V',\] where $p \leq q$ is an expansion rule. To put $X \Rightarrow Y$ into normal form, skip the first step $X_0 \to X_1$. Then derive the rest of the sequence $X_2 \to X_N$ as before. Finally, add the final steps $X_N \to X_{N+1} \to X_{N+2}$ which will first use a ($\pi$-EXP) step to introduce decorated $q$-terms as such: \[\overline{U}, \overline{V} \to \overline{U}, [\pi(q)]^{(2n+1)}, [\pi(q)]^{(2n)}, \overline{V}.\] In the final step, we use (IND-E) to obtain \[\overline{U}', [\pi(q)]^{(2n+1)}, [\pi(q)]^{(2n)}, \overline{V}' \to \overline{U}', [\pi(p)]^{(2n+1)}, [\pi(q)]^{(2n)}, \overline{V}'.\]  

    Now all expansions occur after contractions. Thus the derivation $X_2 \to X_3 \to \ldots \to X_N \to X_{N+1} \to X_{N+2}$ as described above is a normal derivation of length $N$ for $X \Rightarrow Y$. 
    \item[Case IV.] $X_0 \to X_1$ is a (IND-E) step, say $U, [\pi(Y)]^{(2n)}, V \to U, [\pi(Y')]^{(2n)}, V$ where $Y \to Y'$ is an expansion rule.
    \item[Case IV.1.] Suppose there is a (CON) step later in the sequence that cancels out a term produced in the first step. The derivation may take the form 
    \begin{align*}
        U, [\pi(q_0)]^{(2n)}, q_k^{(2n+1)}, V &\to U, [\pi(q_1)]^{(2n)}, q_k^{(2n+1)}, V &&(IND-E)\\
        &\to U, [\pi(q_2)]^{(2n)}, q_k^{(2n+1)}, V &&(\pi-IND)\\
        &\to \cdots && \\
        &\to U, [\pi(q_{k-1})]^{(2n)}, q_k^{(2n+1)}, V &&(\pi-IND)\\
        &\to U, q_k^{(2n)}, q_k^{(2n+1)}, V &&(m-IND)\\
        &\to U, V &&(CON)
    \end{align*}
    where $q_0 \leq q_1$ is an expansion rule and $q_1 \leq q_2 \leq \ldots \leq q_k$ holds in $P$ as a sequence of special rules. This is a derivation of length $k+1$. We can rewrite this as a normal derivation in the following way:
    \begin{align*}
    U, [\pi(q_0)]^{(2n)}, q_k^{(2n+1)}, V &\to U, [\pi(q_0)]^{(2n)}, q_{k-1}^{(2n+1)}, V &&(IND)\\
    &\to U, [\pi(q_0)]^{(2n)}, q_{k-2}^{(2n+1)}, V &&(IND)\\
    &\to \cdots && \\
    &\to U, [\pi(q_0)]^{(2n)}, [\pi(q_{1})]^{(2n+1)}, V &&(\pi-IND)\\
    &\to U, [\pi(q_0)]^{(2n)}, [\pi(q_{0})]^{(2n+1)}, V &&(IND-C)\\
    &\to U,V &&(\pi-CON)
    \end{align*}
    This is a derivation of length $k+1$ which doesn't feature any expansion rules. Together with the induction hypothesis, this will render a normal derivation of $X \Rightarrow Y$. 
    \item[Case IV.2.] Suppose there is a ($\pi$-CON) step featured at the end of the sequence. A derivation of this type can take the form:
    \begin{align*}
       &U, [\pi(q_0)]^{(2n)}, [\pi(q_k)]^{(2n+1)}, V\\
       &\to U, [\pi(q_1)]^{(2n)}, [\pi(q_k)]^{(2n+1)}, V &&(IND-E)\\
       &\to U, [\pi(q_2)]^{(2n)}, [\pi(q_k)]^{(2n+1)}, V &&(\pi-IND)\\
       &\to \cdots &&(\pi-IND)\\
       &\to U, [\pi(q_k)]^{(2n)}, [\pi(q_k)]^{(2n+1)}, V &&(\pi-IND)\\
       &\to U, V &&(\pi-CON)      
    \end{align*}
    This is a derivation of length $k+1$ where $q_0\leq q_1$ is an expansion rule and $q_1\leq q_2 \leq \ldots \leq q_k$ holds in $P$ as a sequence of special rules. We can rewrite this in normal form:
    \begin{align*}
        &U, [\pi(q_0)]^{(2n)}, [\pi(q_k)]^{(2n+1)}, V\\
        &\to U, [\pi(q_0)]^{(2n)}, [\pi(q_{k-1})]^{(2n+1)}, V &&(\pi-IND)\\
        &\to U, [\pi(q_0)]^{(2n)}, [\pi(q_{k-2})]^{(2n+1)}, V &&(\pi-IND)\\   
        &\to \cdots &&(\pi-IND)\\
        &\to U, [\pi(q_0)]^{(2n)}, [\pi(q_{1})]^{(2n+1)}, V &&(\pi-IND)\\
        &\to U, [\pi(q_0)]^{(2n)}, [\pi(q_{0})]^{(2n+1)}, V &&(IND-C)\\
        &\to U, V &&(\pi-CON)
    \end{align*}
    This new sequence has equal length to the previous one but features no expansion rules. Combining this with the induction hypothesis yields a normal derivation of $X \Rightarrow Y$. 
    \item[Case V.] $X_0 \to X_1$ is a (PRE) step. This is neither a contraction rule nor an expansion rule. Since $X_1 \to \ldots X_m$ is a normal derivation, then (PRE) must be a special rule that contributes to a generalised contraction. Hence the derivation is normal.  
\end{enumerate}
\end{proof}

\begin{corollary}
Given a nullable string $X$ in a $\pi$-augmented pregroup, $X \leq 1$ can be verified without using expansion rules.
\end{corollary}
\begin{eg}
Consider the free $\pi$-augmented pregroup $\P$ generated by the poset $P = (\{p,q\}, \to)$ with no relations. The string $pq^{rr}p^\ell q^r$ is not nullable in $\P$, but the string $p[\pi(q)]^{rr}[\pi(p)]^\ell q^r$ is nullable. We have the following derivation:
    \begin{align*}
        p[\pi(q)]^{rr}[\pi(p)]^\ell q^r &\to p [\pi(p)]^r[\pi(q)]^{rr}q^r &&(PRE)\\
        &\to [\pi(q)]^{rr}q^r &&(m-CON)\\
        &\to 1 &&(m-CON)
    \end{align*}
\end{eg}
\section{Application to linguistics}\label{section 4}
\subsection{Clitic movement in Italian}
In many languages, subjects and complements can often attach themselves before or after a verb in the form of clitic pronouns with a potential change in word order.

\begin{eg} Consider the following example sentences in Italian.
\digloss[ex]{Gianni vede Maria.}{Gianni sees Maria.}{}
\digloss[ex]{Gianni la vede.}{Gianni her sees.}{}

In the second sentence, the direct object ``Maria'' is replaced by the pre-verbal clitic pronoun ``la.''
\end{eg}

 Following the literature \cite{casadio2007applying}\cite{casadio2001algebraic}\cite{casadio9tupled} we model Italian grammar using a free pregroup $\mathcal{I}$ generated by the atomic types: $n$, $p$, $s$, $o$, $w$, $\lambda$, which represent nouns, pronouns, sentences, direct objects, indirect objects, and locative phrases, respectively. We denote the partial order by the symbol $\to$. 

Additionally, we consider the types $\overline{o}$, $\overline{w}$, and $\overline{\lambda}$ which represent the clitic forms of direct objects, indirect objects, and locative phrases, respectively. We impose the following equivalences $\overline{o} \Leftrightarrow o$, $\overline{w} \Leftrightarrow w$, $\overline{\lambda} \Leftrightarrow \lambda$ and also the reduction $n \to p$. The latter is particularly significant for pro-drop languages such as Italian.

In this framework, we can analyse our two example sentences to check for grammaticality. As before, a string $X$ is grammatically acceptable if and only if it reduces to the sentence type, i.e. $X \to s$.

For the first sentence, ``Gianni'' is a noun and is thus assigned type $n$, ``vede'' is a transitive verb with type $p^r s o^\ell$, and ``Maria'' is the direct object with type $o$. We have the reduction:
    \begin{align*}
        n (p^r s o^\ell) o &= (n p^r) s (o^\ell o) &&\text{associativity}\\
        &\to (p p^r) s (o^\ell o) &&(IND): n\to p\\
        &\to s (o^\ell o) &&(CON)\\
        &\to s &&(CON)
    \end{align*}
    
In the second sentence, ``la'' is a clitic pronoun replacing a direct object, hence has type $\overline{o}$. The verb ``vede'' must also change its type to $\overline{o}^r p^r s$. 

In \cite{casadio2009clitic} Casadio and Sadrzadeh show that pre-verbal citicisation in Italian satisfies (usually) the left precyclic rule \[qp^\ell \Rightarrow \overline{p}^r q.\]

Since adding this axiom to the list of axioms of a pregroup would collapse the pregroup into a partially ordered group, we can employ the concept of a $\pi$-augmented pregroup to completely formalise clitic movement. This comes with some minor modifications to grammatical typing.

A transitive verb with preverbal cliticisation will be typed $\pi(p)^r \pi(s) \pi(o)^{\ell}$ and its direct object will be $\pi(o)$. 
For a verb with both a direct and an indirect object that may undergo preverbal cliticisation, the type will be $\pi(p)^r \pi(s) \pi(o)^{\ell
}\pi(w)^{\ell}$ and direct and indirect objects will have types $\pi(o)$ and $\pi(w)$ instead of simply $o$ and $w$. 

For post verbal cliticisation, we do not decorate any of the types with the symbol $\pi$ as the precyclic rule is not needed to parse such clitics. 

We now apply $\pi$-augmented pregroups to Italian clitics. The examples are mostly sourced from \cite{casadio2009clitic}. The paper formulated precyclic axioms as ``metarules''. Our method demonstrates that we can incorporate clitic movement into the axioms and derive cliticisation in an algebraic way without appealing to ad-hoc metarules.

First, the sentence ``Gianni vede Maria'' can be transformed into ``Gianni la vede'' via the following sequence of reductions:
    \begin{align*}
        &n \left( \pi(p)^r \pi(s) \pi(o)^\ell \right) \underline{\pi(o)}\\
        &= n \left( \pi(p)^r \pi(s) \pi(o)^\ell \right) (\pi(o)^r)^\ell &&\text{since }(q^r)^\ell = q\\
        &= n \left(\pi(p)^r \pi(s)\right)\left(\underline{\pi(o)^\ell(\pi(o)^r)^\ell}\right) &&\text{associativity}\\
        &= n \left(\underline{\pi(p)^r \pi(s)}\right)\left(\underline{\pi(o)^r \pi(o)}\right)^\ell &&\text{since } p^\ell q^\ell = (qp)^\ell\\
        &\to n \left( \underline{\pi(o)^r \pi(o)}\right)^r \left(\pi(p)^r\pi(s)\right) &&(PRE)\\
        &= n\ \underline{\pi(o)^r \pi(o)^{rr}} \pi(p)^r \pi(s) &&\text{since } (pq)^r = q^rp^r\\
        &\to n\ \pi(\overline{o})^r \pi(\overline{o})^{rr} \pi(p)^r \pi(s) &&(\pi-IND):\overline{o} \Leftrightarrow o
    \end{align*}
    
This concludes the conversion of ``Gianni vede Maria'' to ``Gianni la vede,'' where ``Gianni'' has type $n$, ``la'' has type $\pi(\overline{o})^r$, and ``vede'' has type $\pi(\overline{o})^{rr} \pi(p)^r \pi(s)$. It is straightforward to check that the new sentence is grammatically acceptable. 

\begin{eg}
We consider an example where the verb takes both a direct and an indirect object. 

\trigloss[ex]{Niccolo {\ \ \ \ \ \ \ \ \ \ gives} {a book} {to Ludovica.}}{Niccolo {\ \ \ \ \ \ \ \ \ \ \ \ da} {un libro} {a Ludovica.}}{$\ \ \ \ n$ $\pi(p)^r\pi(s)\pi(w)^\ell\pi(o)^\ell$ $\ \ \ \ \pi(o)$ $\ \ \ \ \ \ \pi(w)$}{}

Here ``un libro'' is the direct object and ``a Ludovica'' is an indirect object. 
\end{eg}

We now apply the logic of $\pi$-augmented pregroups to derive the sentence type where both objects become pre-verbal clitic pronouns. 

\begin{align*}
    &n \pi(p)^r\pi(s)\pi(w)^\ell\pi(o)^\ell\pi(o)\pi(w)\\
    &= n \underline{\pi(p)^r\pi(s)\left(\pi(o)\pi(w)\right)^\ell}\pi(o)\pi(w) &&p^\ell q^\ell = (qp)^\ell\\
    &\to n \left(\pi(o)\pi(w)\right)^r\pi(p)^r\pi(s)\pi(o)\underline{\pi(w)} &&(PRE)\\
    &= n \left(\pi(o)\pi(w)\right)^r\pi(p)^r\pi(s)\underline{(\pi(o)^r)^\ell(\pi(w)^r)^\ell} &&q = (q^r)^\ell\\
    &= n \left(\pi(o)\pi(w)\right)^r\underline{\left(\pi(p)^r\pi(s)\right)\left(\pi(w)^r\pi(o)^r\right)^\ell} &&p^\ell q^\ell = (qp)^\ell\\
    &\to n \left(\pi(o)\pi(w)\right)^r \left(\pi(w)^r\pi(o)^r\right)^r \left(\pi(p)^r\pi(s)\right) &&(PRE)\\
    &= n \pi(w)^r\ \pi(o)^r\ \left(\pi(o)^{rr} \pi(w)^{rr} \pi(p)^r \pi(s)\right) &&\text{pregroup rules}\\
    &\to n \pi(\overline{w})^r\ \pi(\overline{o})^r\ \left(\pi(\overline{o})^{rr} \pi(\overline{w})^{rr} \pi(p)^r \pi(s)\right) &&(\pi-IND): o\Leftrightarrow \overline{o}, w \Leftrightarrow \overline{w}
\end{align*}

The precyclic rule is often performed to substrings rather than individual terms. For example, in the first line, in order to apply $qp^\ell \Rightarrow p^rq$, we take $q = \pi(p)^r\pi(s)$ and $p^\ell = (\pi(o)\pi(w))^\ell$. In the final step, we use the augmented induced step rule starting from $\overline{w}\to w$ to derive $\pi(w)^r \to \pi(\overline{w})^r$, where we use the fact that odd adjoints act contravariantly on the preorder.  We thus obtain the cliticised sentence:

\trigloss[ex]{Niccolo {to her} {\ \ \ it} {\ \ \ \ \ \ \ gives}}{Niccolo {\ \ glie} {\ \ \ lo} {\ \ \ \ \ \ \ \ \ da}}{$\ \ \ \ n$ $\pi(\overline{w})^r$ $\pi(\overline{o})^r$ $\pi(\overline{o})^{rr}\pi(\overline{w})^{rr}\pi(p)^r\pi(s)$}{}

The subject ``Niccolo'' has type $n$, the cliticised indirect object pronoun ``glie''(to her, i.e. to Ludovica) has type $\pi(\overline{w})^r$, the cliticised direct object pronoun ``lo'' (it, i.e. the book) has type $\pi(\overline{o})^r$ and the verb ``da'' has type $\pi(\overline{o})^{rr} \pi(\overline{w})^{rr} \pi(p)^r \pi(s)$. For the sake of completeness, we show that the string associated to the cliticised sentence is grammatically acceptable. \begin{align*}
    &n \pi(\overline{w})^r\ \underline{\pi(\overline{o})^r\ \pi(\overline{o})^{rr}} \pi(\overline{w})^{rr} \pi(p)^r \pi(s)\\
    &\to n \underline{\pi(\overline{w})^r \pi(\overline{w})^{rr}} \pi(p)^r\pi(s) &&(\pi-CON) \\
    &\to \underline{n} \pi(p)^r \pi(s) &&(\pi-CON) \\
    &\to \underline{p \pi(p)^r} \pi(s) &&(IND):n\to p\\
    &\to \pi(s) &&(m-CON)\\
    &\to s &&(m-IND)
\end{align*}

In the last step we apply the mixed induced step rule starting from $s \to s$, which holds by virtue of having taken the reflexive closure of our partial order.

\begin{eg}
Let's consider a sentence with post-verbal cliticisation. 
\trigloss[ex]{Ludovica wants {to see} Chiara.}{Ludovica vuole vedere Chiara.}{$\ \ \ \ n$ $p^rsi^\ell$ $\ \ io^\ell$ $\ \ \ o$}{}
Here, $i$ is the atomic type representing infinitive verbs. Note that none of the types in this string are decorated with the symbol $\pi$. In the case of post-verbal clitics, we can deduce the cliticisation simply from the rule $o \Leftrightarrow \overline{o}$ without the need for precyclic rules. 

The cliticised sentence becomes:

\trigloss[ex]{Ludovica {wants} {to see.her}}{Ludovica {vuole} {veder.la}}{$\ \ \ n$ $p^rsi^\ell$ $i\overline{o}^\ell\overline{o}$}{}

This brings up the question whether transitive verbs need to carry two different types in the pregroup, one where atoms are decorated with the symbol $\pi$ and one where they are not, depending on whether we are dealing with pre-verbal or post-verbal cliticisation. The author proposes that all verbs and their objects be decorated with the symbol $\pi$.
\end{eg}

With this in mind, the sentence ``Ludovica vuole vedere Chiara'' will have type \[n\ \pi(p)^r \pi(s) \pi(i)^\ell\ \pi(i) \pi(o)^\ell \ \pi(o),\]
while the sentence ``Ludovica vuole veder.la'' will have type \[n\ \pi(p)^r \pi(s) \pi(i)^\ell\ \pi(i)\pi(\overline{o})^\ell \pi(\overline{o}).\]

Since $n$ reduces to $p$, we can see, for instance, that the first sentence is grammatically acceptable after applying three contraction rules. Using the precyclic rule anywhere to shift around either of the terms $\pi(i)^\ell$ and $\pi(\overline{o})^\ell$ will render a grammatically unacceptable sentence. Hence decorating all verbal complements by default is a reasonable convention. 

\subsection{Clitic movement in Farsi}
Following \cite{sadrzadeh2007persian}, we model Farsi grammar using the $\pi$-augmented pregroup $(\mathcal{F}, \to)$ which is generated by the atoms $n$, $p$, $o$, $s$, $q$, which represent the types of nouns, personal pronouns, direct objects, sentences, and questions, respectively. We also add $\overline{o}$ and $\overline{p}$ to represent the clitic forms of direct objects and personal pronouns, and we impose the equivalences: $\overline{p} \Leftrightarrow p$, $\overline{o} \Leftrightarrow o$.

Clitic movement in Farsi follows right precyclic rules. So in our $\pi$-augmented pregroup, the rule (PRE) is replaced by the rule \[\text{(R-PRE)}\quad \quad  X, [\pi(Y)]^r, \pi(Y'), Z \to X, \pi(Y'), [\pi(Y)]^\ell, Z.\]

\begin{eg}
    Consider the following example of cliticisation in Farsi. The first sentence contains both the subject ``Hassan'' and the direct object ``Nadia'' explicitly. The second sentence is an example of partial cliticisation where ``Nadia'' gets replaced by the post-verbal clitic ``ash.'' The third sentence contains the fully cliticised sentence where both the subject and the object are substituted by the post-verbal clitic pronouns ``d'' and ``ash.''
    \trigloss[ex]{Hassan Nadia-ra {   did}}{Hassan Nadia {   saw}}{$\ \ \ \ p$ $\ \ \ o$ $(o^rp^rs)$}{}
    \trigloss[ex]{Hassan did ash}{Hassan saw her}{$p$ $p^{r}s\overline{o}^\ell$ $\overline{o}$}{}
    \trigloss[ex]{di d ash}{saw he her}{$(s\overline{o}^{\ell}\overline{p}^{\ell})$ $\ \overline{p}$ $\ \overline{o}$}{}
\end{eg}

We can derive the partially cliticised sentence using the logic of $\pi$-augmented pregroups with the following sequence of derivations. 
\begin{align*}
    &\pi(p) \underline{\pi(o)} \left(\pi(o)^r \pi(p)^r \pi(s)\right)\\
    &= \pi(p) \underline{(\pi(o)^\ell)^r \pi(o)^r} \pi(p)^r \pi(s) &&\pi(o) = (\pi(o)^\ell)^r\\
    &= \pi(p) \underline{\left(\pi(o)\pi(o)^\ell\right)^r} \underline{\left(\pi(p)^r\pi(s)\right)}&& p^rq^r = (qp)^r\\
    &\to \pi(p) \left(\pi(p)^r \pi(s)\right) \left(\pi(o)\pi(o)^\ell\right)^\ell &&(R-PRE)\\
    &= \pi(p) \left(\pi(p)^r \pi(s) \underline{\pi(o)}^{\ell\ell}\right) \underline{\pi(o)}^\ell  &&\text{associativity}\\
    &\to \pi(p) \left(\pi(p)^r \pi(s) \pi(\overline{o})^{\ell\ell}\right) \pi(\overline{o})^\ell &&(\pi-IND): o \Leftrightarrow \overline{o} 
\end{align*}

Here the subject ``Hassan'' remains typed as $\pi(p)$, the verb ``did'' has type $\pi(p)^r \pi(s) \pi(\overline{o})^{\ell\ell}$, and the cliticised direct object ``ash'' has type $\pi(\overline{o})^\ell$. 

Following a similar procedure, we can derive the fully cliticised sentenced ``di d ash'' where the verb ``di'' has type $\pi(s)\pi(\overline{o})^{\ell\ell}\pi(\overline{p})^{\ell\ell}$, the cliticised subject ``d'' has type $\pi(\overline{p})^\ell$, and the cliticised object ``ash'' has type $\pi(\overline{o})^\ell$. 
\section{Conclusion}
We introduced the concept of a $\pi$-augmented pregroup, where decorated substrings of any given string satisfy precyclic axioms. This allows one to include precyclicity into the axioms of a pregroup without reducing the pregroup to a partially ordered group. We proved a normalisation theorem akin to the Switching Lemma and we formalised some applications of precyclic permutations applied to clitic movement in natural languages. 

It would be interesting to see if such an approach could be used to formalise more generally other metarules. For instance, it would be useful to see a formalisation of word order flexibility in Hindi syntax as it was discussed in \cite{debnath2019pregroup} in terms of $\pi$-augmented pregroups. The main target would be to formalise what is described in the reference as ``blocking'' certain precyclic permutations. 

The author is currently working on a cross-linguistic analysis of clitic movement in Eastern Romance languages, where the concept of a $\pi$-augmented pregroup acts as a reasonable underlying theoretical model. 

%%%%%BIBLIOGRAPHY
\bibliography{article}
\bibliographystyle{alpha}
\end{document}